\documentclass{amsart}
\usepackage{max, mathpartir}
\usepackage[square,sort,comma,numbers]{natbib}
\usepackage[frozencache,cachedir=.]{minted}

\newcommand*{\xsuppose}{\mathbf{suppose}\mathop{}}
\newcommand*{\ysuppose}{\mathbf{suppose}}
\newcommand*{\xif}{\mathbf{if}\mathop{}}
\newcommand*{\yif}{\mathbf{if}}
\newcommand*{\xthen}{\mathop{}\mathbf{then}\mathop{}}
\newcommand*{\xelse}{\mathop{}\mathbf{else}\mathop{}}
\newcommand*{\yelse}{\mathbf{else}}
\newcommand*{\xcase}{\mathbf{case}\mathop{}}
\newcommand*{\ycase}{\mathbf{case}}
\newcommand*{\xof}{\mathop{}\mathbf{of}\mathop{}}
\newcommand*{\xend}{\mathop{}\mathbf{end}}
\newcommand*{\xbool}{\mathbf{Bool}}
\newcommand*{\xtrue}{\mathbf{True}}
\newcommand*{\xfalse}{\mathbf{False}}
\newcommand*{\xBoolTwo}{\mathbf{Bool2}}
\newcommand*{\xPair}{\mathbf{Pair}}

\makeatletter
\newcommand{\mask}[2]{{\mathpalette\mask@{{#1}{#2}}}}
\newcommand{\mask@}[2]{\mask@@{#1}#2}
\newcommand{\mask@@}[3]{%
  \settowidth{\dimen@}{$\m@th#1#2$}%
  \makebox[\dimen@]{$\m@th#1#3$}%
}
\makeatother

\title{Genetic Programming with Local Scoring}
\author{Max Vistrup}
\date{November 11, 2022}

\begin{document}

\begin{abstract}
    We present new techniques for synthesizing programs
        through sequences of mutations.
    Among these are
        (1) a method of \emph{local scoring}
                assigning a score to each expression in a program,
            allowing us to more precisely identify buggy code,
        (2) \emph{\texttt{suppose}-expressions}
            which act as an intermediate step to
                evolving \texttt{if}-conditionals, and
        (3) \emph{cyclic evolution}
            in which we evolve programs through phases of expansion and reduction.
    To demonstrate their merits,
        we provide a basic proof-of-concept implementation
        which we show evolves correct code for several functions
            manipulating integers and lists,
            including some that are intractable
                by means of existing Genetic Programming techniques.
\end{abstract}

\maketitle
\tableofcontents

\section{Introduction}
\label{sec:introduction}

From Darwinian evolution, a process of local optimization,
    emerged highly complex organisms.
An organism is, in a sense, a program encoded by DNA
    determining physical actions based on observed data.
It is therefore natural to ask
    whether it would also be possible to evolve through natural selection
        computer code that solves a given problem.
This is the task at the heart of Genetic Programming.

This paper is aimed at describing a hill-climbing algorithm for
    evolving programs,
    taking an approach that is quite different
        from established Genetic Programming methods.
As a starting point,
    let us first sketch an extremely simple hill-climbing algorithm
        for evolving programs for a programming language $\mathscr L$.
For that algorithm,
    we will need
    \begin{enumerate}[(1)]
        \item
        an initial program $P_0 \in \mathscr L$,\footnote{
            For example, this could be a randomly generated program
                or a program that we are trying to debug.
        }
        \item
        a scoring function $S : \mathscr L \to [0, 1]$
            assessing how well a program is performing at solving a given task
            with $1$ meaning perfect,\footnote{
                $S$ will typically work by
                    running the program on a pre-defined list of inputs and
                    then comparing the actual outputs to the desired outputs,
                    assigning a score based on how many of these cases are correct.
            } and
        \item
        for each program $P \in \mathscr L$,
            a probability distribution $M(P)$ on $\mathscr L$
            of \q{mutations} of $P$.\footnote{
        A good choice of the distribution $M(P)$ should minimize
            the expected score loss
            $
                L(P) \defeq S(P) - \mathbb E[S(M(P))].
            $
        In particular, likely mutations should ideally
            only affect the program behavior slightly.
        However, small changes in code often have major effect on the behavior of the program,
            so typically $L > 0$.
    }
    \end{enumerate}
Given this, the (non-deterministic) algorithm produces
    a sequence of programs $P_0, P_1, \ldots \in \mathscr L$
    in which $P_0$ is the initial program and for $k \geq 1$,
    \[
        P_k \defeq
        \begin{cases}
            P_k',
                & \text{if $S(P_k') > S(P_{k - 1})$;} \\
            P_{k - 1},
                & \text{otherwise;}
        \end{cases}
    \]
    where $P_k'$ is randomly sampled from $M(P_{k - 1})$.
The algorithm terminates with solution $P_k$ when $S(P_k) = 1$,
    which may or may not happen eventually.
Put simply, the algorithm attempts to improve a program one small step at a time
    until it arrives at a satisfying solution,
    quite similar to
        how a human might approach writing or debugging a program.

This paper describes some inadequacies of this primitive idea
    and how to overcome them,
    culminating in a feasible algorithm for evolving programs along these lines.
Our investigation is structured as follows.
\begin{enumerate}[(1)]
    \item
    In \cref{sec:local},
        we shall see how assign a score to each expression in a program
            as opposed to the entire program at once.
    In this way,
        we get a more precise understanding of which parts of a program are buggy.
    Later, these \q{local scores} are leveraged to choose mutations so that
        the expressions that are the most buggy are
            the most likely to be mutated.
    \item
    In \cref{sec:suppose},
        we introduce $\ysuppose$-expressions.
    These are expressions that are waiting to be promoted into $\yif$-conditionals.
    This two-stage process gives an efficient and precise way
        to evolve conditions for $\yif$-conditionals.
    \item
    In \cref{sec:principles},
        we discuss greedy criteria
            for when a mutation should be regarded successful.
    Specifically, we will mutate one expression at a time,
        and the mutation is regarded successful
            if the score of the new expression is
                strictly greater than the score of the old expression.
    \item
    In \cref{sec:cyclic},
        we describe an evolutionary process which consists of \emph{cycles}.
    Each cycle is divided into four stages:
        stretching, mutation, rewinding, and compression.
    This design allows us to probe new program structures
        while avoiding endless program expansion.
    \item
    In \cref{sec:poc},
        we present a basic proof-of-concept implementation of these ideas and
        show that it is able to generate correct code for several small problems.
    On one well-known problem, the even-parity problem,
        we also compare our implementation to some found in the literature
        and find that our methods outperforms the best of the considered implementations
            by more than an order of magnitude.
    \item
    In \cref{sec:future},
        we discuss future theoretical and practical directions,
        including integration with Deep Learning techniques.
\end{enumerate}
%
%

\subsection{Functional Programming as a setting for Genetic Programming}
This paper is focused on functional programs,
    although some of its techniques are applicable to other paradigms as well.
The lack of side effects lends itself to Genetic Programming
    as it means that changing a part of the code has
        a more predictable effect on the program behavior:
    we need not worry about how a state change could affect
        the behavior of expressions elsewhere in the program.

\section{Local Scoring}
\label{sec:local}

When evolving programs,
    we need to be able to quantify how well a program is performing.
As we discussed in the introduction,
    one could do that by assigning the program a score
        based on how well it adheres to its test cases.
However, with such an approach to scoring,
    if a test case fails,
    we gain little insight into what parts of the code could be the root cause of the failure.
If we were able to trace back such failure to specific parts of the code,
    we could focus on mutating those parts.

In this section, we describe a simple method for assigning each expression a score.
We assume that every type is an Algebraic Data Type
    so that values can be viewed as trees;
    each node corresponds to a value constructor
        and its children correspond to its arguments.
Then, the idea is to augment each node in the value tree with a \emph{trace},
    which is a list of expressions that could affect the value at that node.
To state this method precisely,
    we exhibit a rudimentary programming language
    implementing this idea.

\subsection{A programming language with traced values}
\label{subsec:language}
As the setting for this paper,
    we consider a simplistic, functional programming language
    whose values are \q{traced}.
We first describe its syntax:
    \begin{align*}
        f
        \Coloneqq{}& g_1 \mid g_2 \mid \cdots
        &
        \text{function}
        \\
        c
        \Coloneqq{}& \xtrue \mid \xfalse \mid d_1 \mid d_2 \mid \cdots
        &
        \text{value constructor}
        \\
        x
        \Coloneqq{}& y_1 \mid y_2 \mid \cdots
        &
        \text{variable}
        \\
        t
        \Coloneqq{}& 1 \mid 2 \mid \cdots
        &
        \text{expression tag}
        \\
        E
        \Coloneqq{}&
        x \mid
        c(e, \ldots, e) \mid
        f(e, \ldots, e)
        &
        \\
        \mid{}&
        \xcase e \xof
            c(x, \ldots, x) \to e;
            \cdots;
            c(x, \ldots, x) \to e
        \xend
        \\
        \mid{}&
        \xif e \xthen
            e
        \xelse
            e
        \xend
        \\
        \mid{}&
        \xsuppose e \xthen
            e
        \xend
        \\
        e
        \Coloneqq{}&
        E^{(t)}
        &
        \text{tagged expression}
        \\
        D
        \Coloneqq{}&
        f(x, \ldots, x) = e
        &
        \text{function definition}
        \\
        P
        \Coloneqq{}&
        D \mid D; P
        &
        \text{program}
    \end{align*}
We stipulate that
    \begin{enumerate}[(1)]
        \item
        in a program $P$,
            each function $g_i$ may be declared at most once,
        \item
        in a program $P$,
            each participating expression has a unique tag,
            that is,
            among all the bodies of functions and all their subexpressions,
            each tag $i$ appears at most once,
        \item
        cases in a $\ycase$-expression should be disjoint,
            that is,
            the constructors appearing on the left-hand side of the arms
                should be pairwise distinct,
        \item
        a variable can be bound at most once in each case of a $\ycase$-expression,
            that is,
            the left-hand side of the arm takes the form $c(x_1, \ldots, x_n)$
                where $x_1, \ldots, x_n$ are pairwise distinct variables, and
        \item
        variables bound by a case of a $\ycase$-expression cannot
            overshadow variables that are already bound.
    \end{enumerate}
We now describe the typing rules.
We use a simple nominal type system with types
    \[
        \tau \Coloneqq \xbool \mid \tau_1 \mid \tau_2 \mid \cdots.
    \]
A \emph{typing context} $\Gamma$ is a collection of judgments of the forms
    \begin{enumerate}[(1)]
        \item
        a variable typing:
        $x : \tau$,
        \item
        a function signature:
        $f : (\tau, \ldots, \tau) \to \tau$, or
        \item
        a type definition:
        $\tau \Coloneqq c(\tau, \ldots, \tau) \mid \cdots \mid c(\tau, \ldots, \tau)$.
    \end{enumerate}
The typing rules are
\[
    \inferrule[TrueType]
        {}
        {\Gamma \vdash \xtrue^{(t)} : \xbool}
    \qquad
    \inferrule[FalseType]
        {}
        {\Gamma \vdash \xfalse^{(t)} : \xbool}
    \quad
    \inferrule[VarType]
        {x : a \in \Gamma}
        {\Gamma \vdash x^{(t)} : a}
\]
\[
    \inferrule[FunctionType]
        {
            \Gamma \vdash e_1 : a_1, \ldots, e_n : a_n
            \\
            f : (a_1, \ldots, a_n) \to b \in \Gamma
        }
        {\Gamma \vdash f(e_1, \ldots, e_n)^{(t)} : b}
\]

\[
    \inferrule[ConstructorType]
        {
            \Gamma \vdash e_1 : a_1, \ldots, e_n : a_n
            \\
            a \Coloneqq \cdots \mid c(a_1, \ldots, a_n) \mid \cdots \in \Gamma
        }
        {\Gamma \vdash c(e_1, \ldots, e_n)^{(t)} : a}
\]

\[
    \inferrule[CaseType]
        {
            a \Coloneqq
                c_1(a_{1, 1}, \ldots, a_{1, n_1})
                \mid
                \cdots
                \mid
                c_m(a_{m, 1}, \ldots, a_{m, n_m})
            \in \Gamma
            \\
            \Gamma \vdash
            d : a
            \\
            \Gamma, \forall j. x_{1, j} : a_{1, j} \vdash e_1 : b
            \\
            \cdots
            \\
            \Gamma, \forall j. x_{m, j} : a_{m, j} \vdash e_m : b
        }
        {
            \Gamma
            \vdash
            \xcase d \xof
                c_1(x_{1, 1}, \ldots, x_{1, n_1}) \to e_1;
                \cdots;
                c_m(x_{m, 1}, \ldots, x_{m, n_m}) \to e_n
            \xend^{(t)}
            : b
        }
\]

\[
    \inferrule[IfType]
        {\Gamma \vdash e_1 : a, e_2 : a, c : \xbool}
        {
            \Gamma \vdash
            \xif c \xthen
                e_1
            \xelse
                e_2
            \xend^{(t)}
            : a
        }
    \qquad
    \inferrule[SupposeType]
        {\Gamma \vdash e : a, c : \xbool}
        {
            \Gamma \vdash
            \xsuppose c \xthen
                e
            \xend^{(t)}
            : a
        }
\]
We shall often write just $P$ when we mean
    a well-typed program $(P, \Gamma)$,
    letting the typing context $\Gamma$ containing all the
        type definitions and
        function signatures
        be implicit.

At last, we explain the semantics of our programming language.
The values of the programming language are
    \[
        v \Coloneqq c^{[\pm t, \ldots, \pm t]}(v, \ldots, v)
    \]
    where $\pm \Coloneqq + \mid -$.
Conceptually, each value constructor is tagged with a \emph{trace},
    a list of expressions that could affect the value.
Its trace need not contain the traces of the arguments $(v, \ldots, v)$
    whose traces in turn consist of expressions
        that could affect the arguments but not the parent value constructor.
A sign $\pm$ should be thought of as measuring the \q{influence} of an expression:
    $+$ means that if the value is bad then the expression should be punished, and
    $-$ means that if the value is bad then the expression should be rewarded.
The relevance of $-$ will be revealed in \cref{sec:suppose}.

An \emph{evaluation context} $\sigma$ is a collection of bindings of the form
    \[
        B \Coloneqq x = v \mid f(x, \ldots, x) = e
    \]
    such that each variable $x$ and each function $f$
        appears at most once in the left-hand sides.
For a program $P$,
    one extracts an evaluation context $\sigma(P)$ by collecting all its function definitions.

We use the notation $\sigma \vdash e \downarrow v$ to denote that
    an expression $e$
    evaluates to a value $v$
    in the evaluation context $\sigma$.
Furthermore, we make use of the notation
    \[
        \(c^{[t_{k + 1}, \ldots, t_n]}(v_1, \ldots, v_k)\)^{[t_0, \ldots, t_k,]}
        \defeq
        c^{[t_0, t_1, \ldots, t_n]}(v_1, \ldots, v_k)
    \]
    where $t_i$ are signed tags.
We shall sometimes elide the $+$ in $+ t$ for notational convenience.
The rules for evaluation are as follows.
\[
    \inferrule[Var]
        {}
        {\sigma, x = v \vdash x^{(t)} \downarrow v^{[t,]}}
    \qquad
    \inferrule[Cons]
        {\sigma \vdash e_1 \downarrow v_1, \ldots, e_n \downarrow v_n}
        {\sigma \vdash c(e_1, \ldots, e_n)^{(t)} \downarrow c^{[t]}(v_1, \ldots, v_n)}
\]
\[
    \inferrule[Call]
        {
            \sigma \vdash e_1 \downarrow v_1, \ldots, e_n \downarrow v_n
            \\
            x_1 = v_1, \ldots, x_n = v_n \vdash e \downarrow v
        }
        {\sigma, f(x_1, \ldots, x_n) = e \vdash f(e_1, \ldots, e_n)^{(t)} \downarrow v^{[t,]}}
\]
\[
    \inferrule[Case]
        {
            \sigma \vdash d \downarrow c^{[t_1, \ldots, t_k]}(v_1, \ldots, v_n)
            \\
            \sigma, x_1 = v_1, \ldots, x_n = v_n \vdash e \downarrow v
        }
        {
            \sigma \vdash
            \xcase d \xof
                \cdots;
                c(x_1, \ldots, x_n) \to e;
                \cdots
            \xend^{(t_0)}
            \downarrow v^{[t_0, t_1, \ldots, t_k,]}
        }
\]
\[
    \inferrule[IfTrue]
        {
            \sigma \vdash
            c \downarrow \xtrue^{[t_1, \ldots, t_n]},
            e_1 \downarrow v
        }
        {
            \sigma \vdash
            \xif c \xthen
                e_1
            \xelse
                e_2
            \xend^{(t_0)}
            \downarrow v^{[t_0, t_1, \ldots, t_n,]}
        }
\]
\[
    \inferrule[IfFalse]
        {
            \sigma \vdash
            c \downarrow \xfalse^{[t_1, \ldots, t_n]},
            e_2 \downarrow v
        }
        {
            \sigma \vdash
            \xif c \xthen
                e_1
            \xelse
                e_2
            \xend^{(t_0)}
            \downarrow v^{[t_0, t_1, \ldots, t_n,]}
        }
\]

We postpone stating the semantics of $\ysuppose$-expressions to \cref{sec:suppose}.

Note that whenever an expression is evaluated,
    the value it produces is tagged with a unique identifier for that expression.
The rules are defined in such a way that
    each value ends up carrying traces of expressions that could affect the value.

To illustrate, let us elaborate a bit on the \textsc{Case} rule.
The rule says that what a $\ycase$-expression
    evaluates to may be affected by
    \begin{enumerate}[(1)]
        \item
        the $\ycase$-expression $t_0$ itself since, trivially,
            replacing the entire $\ycase$-expression could change its value,
        \item
        those expressions $t_1, \ldots, t_n$ that
            could affect the entirety, even the value constructor,
            of the value attained by $d$ that we are casing on, and
        \item
        those expressions that could affect the value attained by
            the body $e$ of the chosen arm.
    \end{enumerate}
The rules \textsc{IfTrue} and \textsc{IfFalse} are quite similar.\footnote{
        The reason that we include $\yif$-conditionals
            in spite of having $\ycase$-expressions
            is that they simplify our discussion of $\ysuppose$-expressions
                in \cref{sec:suppose}.
    }

\subsection{Extracting local scores}
\label{subsec:local_score_algo}
We shall need to consider values with other kinds of tags.
Define
    \[
        v^u \Coloneqq c^{u}(v^u, \ldots, v^u).
    \]
Of particular use are
    untagged values $v^{()}$,
    traced values $v = v^{[\pm t, \ldots, \pm t]}$, and
    score-tagged values $v^{[0, 1]}$.

Fix a base typing context $\Gamma_{\f{types}}$
    containing just the type definitions we shall use.
To assign scores to expressions in a program $(P, \Gamma)$
    with $\Gamma_{\f{types}} \subseteq \Gamma$,
    we will need a \emph{test criterion}
    consisting of the data
    \begin{enumerate}[(1)]
        \item
        a name and signature $f : (a_1, \ldots, a_n) \to b$ for the function we are testing,
            where $a_1, \ldots, a_n, b$ are types defined in $\Gamma_{\f{types}}$,
        \item
        untagged input values $(v_{1, k} : a_1, \ldots, v_{n, k} : a_n)$ for $k = 1, \ldots, N$, and
        \item
        for each $k = 1, \ldots, N$, an \emph{assessment function}
            \[
                s_{(v_{1, k}, \ldots, v_{n, k})} :
                \{w \in v^{()} \mid w : b\}
                \to
                \{w \in v^{[0, 1]} \mid w : b\}
            \]
            which takes a $b$-typed value
                and sets the tag at each node to a score
                (i.e., does not change the underlying untagged value)
            such that all the tags in
                $
                    s_{(v_{1, k}, \ldots, v_{n, k})}(w)
                $
                are $1$
                if and only if $w$ is the desired result of $f(v_{1, k}, \ldots, v_{n, k})$.\footnote{
                    Moreover,
                        scores above $0.5$ ought to represent \q{closer to right than to wrong} and
                        scores below $0.5$ \q{closer to wrong than to right}.
                    (This will later become important for $\ysuppose$-expressions.)
                }
    \end{enumerate}
We call the function specified in (1) the \emph{primary function}.

For example,
    suppose $P$ consisted of
        just one function definition $\text{swap}(x, y) = \cdots$ and
    $\Gamma$ consisted of
        the type definition
            $\xBoolTwo \Coloneqq \xPair(\xbool, \xbool)$ and
        the function signature
            $\text{swap} : \xBoolTwo \to \xBoolTwo$.
If we were to assess how well $\text{swap}$ swaps the booleans in the pair
    like $(x, y) \mapsto (y, x)$,
    the test criterion could consist of
    \begin{itemize}
        \item[($1'$)]
        the name and signature
            $
                \text{swap} : \xBoolTwo \to \xBoolTwo,
            $
        \item[($2'$)]
        the inputs
            $\xPair^{()}(\xtrue^{()}, \xtrue^{()})$,
            $\xPair^{()}(\xfalse^{()}, \xtrue^{()})$,
            \\
            $\xPair^{()}(\xtrue^{()}, \xfalse^{()})$,
            $\xPair^{()}(\xfalse^{()}, \xfalse^{()})$, and
        \item[($3'$)]
        the assessment functions
            \[
                s_{\xPair(a_1^{()}, a_2^{()})} :
                \xPair^{()}(b_1^{()}, b_2^{()})
                \mapsto
                \xPair^{(\delta(b_1, a_2) + \delta(b_2, a_1))/2}
                    (b_1^{\delta(b_1, a_2)}, b_2^{\delta(b_2, a_1)})
            \]
            where $\delta(x, y)$ is $1$ if $x = y$ and $0$ otherwise.
    \end{itemize}
The assessment functions are defined in such a way
    that the resulting tag to the $\xPair$ constructor measures
        how close the value is to the desired value $\xPair(a_2^{()}, a_1^{()})$,
    and the resulting tag to the boolean in each coordinate measures
        how close the respective coordinate is to the desired value.
Notice how the assessment functions give perfect scores
    when the behavior is as desired:
    \[
        s_{\xPair(a_1^{()}, a_2^{()})}(\xPair^{()}(a_2^{()}, a_1^{()}))
        =
        \xPair^{1}(a_2^{1}, a_1^{1}).
    \]

Let us return to the general setup $(1)$-$(3)$
    and explain how to compute a score for each expression in the program $P$.
We assume that $P$ contains such an $f$ and
    that $f$ always terminates.\footnote{
        The latter restriction can be dealt with
            by imposing limits to the number of execution steps allowed
            or by restricting what kinds of recursion are allowed.
        The minutiae are dealt with in our proof-of-concept implementation;
            see \cref{subsec:rec}.
    }
The process for obtaining expression-level scores is as follows.
\begin{enumerate}[(a)]
    \item
    For each $k = 1, \ldots, N$,
        evaluate the function call
        $
            f(\f{tr}(v_{1, k}), \ldots, \f{tr}(v_{n, k}))
        $
        in $P$
        and call its (traced) value $w_k$.
    Here, $\f{tr}(v)$ is the traced value obtained from the untagged value $v$ by
        setting each tag to the empty trace $[]$.
    \item
    For each $k = 1, \ldots, N$,
        compute
        $
            w_k' \defeq s_{(v_{1, k}, \ldots, v_{n, k})}(\f{untag}(w_k))
        $
        where $\f{untag}(w_k)$ is the untagged value obtained from the traced value $w_k$ by
            replacing each tag with $()$.
    \item
    For each node in $w_k$, extract
        the positively charged tags $+ t_1, \ldots, + t_p$ from its trace,
        the negatively charged tags $- u_1, \ldots, - u_q$ from its trace, and
        the tag $s \in [0, 1]$ of the corresponding node in $w_k'$;
        then collect all the tuples
            $(t_i, s)$ for $i = 1, \ldots, p$ and
            $(u_i, 1 - s)$ for $i = 1, \ldots, q$
            into a list $\ell$.
    \item
    An expression in $P$ with the tag $t$ is assigned the score,
        \[
            S_P(t) \defeq
                \frac
                    {\sum_{(t, s) \in \ell} s}
                    {\sum_{(t, s) \in \ell} 1}
        \]
        averaging over all the scores in $\ell$ associated to the tag $t$.
\end{enumerate}
Note how
    a positively charged tag $+t$ in some trace in $w_k$ contribute to the average defining $S_P(t)$
        with the score produced by
        $
            s_{(v_{1, k}, \ldots, v_{n, k})}
        $
        corresponding to the trace
    whereas a negatively charged tag $- t$ contribute
        with the inversion of that score.
The net effect is that
    an expression with tag $t$ that appears in a trace positively charged
        (resp.~negatively charged)
        is incentivized to maximize (resp.~minimize) the score corresponding to that trace
        if we are to maximize the score $S_P(t)$.

For the rest of the paper,
    we shall fix a test criterion (with respect to the fixed $\Gamma_{\f{types}}$).
Moreover, from now on,
    we will only consider programs $P$
        that contain
            the type definitions $\Gamma_{\f{types}}$ and
            a function definition for the primary function (1).
In particular, for each such program $P$,
    we get a scoring function $S_P$.
If $e$ is an expression in $P$ with expression tag $t$,
    we write $S_P(e) \defeq S_P(t)$.
Also, when the context resolves ambiguity,
    we shall omit the subscript $P$,
    writing $S$ instead of $S_P$.

\section{\texttt{suppose}-expressions}
\label{sec:suppose}

Let us informally consider the problem of evolving code for
    the minimum function
    $
        \f{min} : \Z \times \Z \to \Z.
    $
We may assess a candidate implementation
    based on some example arguments $A \subseteq \Z \times \Z$
    ranging over small numbers.
Start with the initial program
    \[
        \tag{$P_0$}
        \f{min}(n, m) = 0.
    \]
This is not faring particularly well:
    perhaps,
    there are a few $(n, m) \in A$ with $0 = n \leq m$ or $0 = m \leq n$
    but they make up an increasingly small fraction as $|A|$ grows.
If we try to mutate $P_0$ by replacing the body of $\f{min}(n, m)$
    by various randomly chosen expressions and testing it on $A$,
    it is not hard to find the improved program
    \[
        \tag{$P_1$}
        \f{min}(n, m) = n,
    \]
    which performs a lot better:
    it is correct whenever $n \leq m$
        which is roughly $\frac12$ of the time.
However, at this point, we stagnate.
The next program we want is
    \[
        \tag{$P_2$}
        \f{min}(n, m) = \xif n \leq m \xthen n \xelse m \xend,
    \]
    but this is quite a leap from $P_1$ and
    would therefore take a long time
        to discover through picking out random expressions.
More complex examples are even less tractable.
How do we overcome this problem of evolving $\yif$-conditionals?

$\ysuppose$-expressions allow us to separate
    the problem of evolving the condition
    from the problem of evolving the $\yelse$-branch.
In the above informal example,
    small intermediate steps between $P_1$ and $P_2$ could be
    \begin{align*}
        &\f{min}(n, m) = n
        \\
        \becomes{}&
        \f{min}(n, m) = \xsuppose \xtrue \xthen n \xend
        \\
        \becomes{}&
        \f{min}(n, m) = \xsuppose n \leq m \xthen n \xend
        \\
        \becomes{}&
        \f{min}(n, m) = \xif n \leq m \xthen n \xelse n \xend
        \\
        \becomes{}&
        \f{min}(n, m) = \xif n \leq m \xthen n \xelse m \xend,
    \end{align*}
Conceptually, $\ysuppose$-expressions encourage the condition,
    initially constant $\xtrue$,
    to evolve to be $\xtrue$ whenever the body is correct and $\xfalse$ otherwise.
When the condition for the $\ysuppose$-expression is good enough,
    we will promote the $\ysuppose$-expression into an $\yif$-conditional:
    \[
        \xsuppose c \xthen e \xend
        \quad
        \becomes
        \quad
        \xif c \xthen e \xelse e \xend.
    \]
This evolutionary strategy is established in \cref{sec:cyclic}.

With this discussion in mind,
    let us now state the semantics of $\ysuppose$-expressions.
We use the notation established in \cref{sec:local}.
Define ${-}{+}t \defeq -t$ and ${-}{-}t \defeq +t$.
The rules are then as follows.
\[
    \inferrule[SupposeTrue]
        {
            \sigma \vdash
            c \downarrow \xtrue^{[t_1, \ldots, t_n]},
            e \downarrow v
        }
        {
            \sigma \vdash
            \xsuppose c \xthen
                e
            \xend^{(t_0)}
            \downarrow v^{[t_0, t_1, \ldots, t_n,]}
        }
\]
\[
    \inferrule[SupposeFalse]
        {
            \sigma \vdash
            c \downarrow \xfalse^{[t_1, \ldots, t_n]},
            e \downarrow v
        }
        {
            \sigma \vdash
            \xsuppose c \xthen
                e
            \xend^{(t_0)}
            \downarrow v^{[t_0, -t_1, \ldots, -t_n,]}
        }
\]
Thus, if we disregard traces,
    $\ysuppose$-expressions just evaluate their bodies.
However, if we follow the Local Scoring model set out in \cref{sec:local},
    we see that if $c$ evaluates to $\mathbf{True}$,
    the better $e$ performs,
    the higher score is assigned to $c$,
    whereas if $c$ evaluates to $\xfalse$,
    the better $e$ performs,
    the lower score is assigned to $c$.

We will need the following lemma in \cref{sec:cyclic}.
\begin{lemma}\label{lem:suppose_true}
    Let $P$ be a program.
    Suppose $P$ contains an expression
        $
            \xsuppose \xtrue^{(t)} \xthen
                e
            \xend^{(u)}.
        $
    Then $S(t) = S(u)$.
\end{lemma}

\begin{proof}
    Recall the algorithm specified in \cref{subsec:local_score_algo}.
    Note that the steps (a)-(c) are the same
        no matter whether we are computing $S(t)$ or $S(u)$.
    From the rule \textsc{SupposeTrue},
        we see that during step (a),
        $+u$ (resp.~$-u$) will always appear coupled with $+t$ (resp.~$-t$) in traces.
    It follows that
        $\ell$ has exactly $m$ pairs $(t, s)$
            if and only if it has exactly $m$ pairs $(u, s)$.
    Therefore,
        \[
            S(t)
            =
            \frac
                {\sum_{(t, s) \in \ell} s}
                {\sum_{(t, s) \in \ell} 1}
            =
            \frac
                {\sum_{(u, s) \in \ell} s}
                {\sum_{(u, s) \in \ell} 1}
            =
            S(u),
        \]
        completing the proof.
\end{proof}

\section{Some mutation principles}
\label{sec:principles}
We discuss here some simple principles for how to mutate programs
    by leveraging Local Scoring.
These will be used in combination with
    the methods presented in \cref{sec:cyclic}.

\subsection{Mutually exclusive mutation}
\label{subsec:mutex}
Changing a single expression has a high chance of breaking the program,
    but changing multiple expressions has an even higher chance of doing so.
To deal with this fragility,
    we shall change just one expression
    in each mutation step.

A simple scheme to mutate a program $P$ is to first pick a random expression.
The probability for picking a given expression $e$
    should be proportional to $1 - S(e)$,
    so that
        we are more likely to mutate what is broken and
        we never mutate what works perfectly.\footnote{
            In particular, under this scheme, dead code is never mutated.
        }
This expression is then replaced by a new, randomly generated expression of the same type
    (with expression tags disjoint from the expression tags used in $P$),
    giving a new program $P'$, the mutant of $P$.

\subsection{Greedy Principle of Mutation}
\label{subsec:greedy}
Suppose we use a hill-climbing algorithm
    which has as state a program $P$
        that it repeatedly tries to mutate
        until it finds a good mutant, a successor,
        which then replaces $P$,
    repeating this process until the program solves some task perfectly.
How do we determine if a given mutant $P'$ should replace $P$ as its successor?

If we go purely by some kind of global score based on the behavior of the program,
    as in \cref{sec:introduction},
    we fail to evolve the conditions for $\ysuppose$-expressions
        changing which has no exterior effect on the program behavior.
We therefore propose another criterion.
Recall that we mutate a program $P$ to a program $P'$
    by replacing some expression $e$ by a new expression $e'$.
The \emph{Greedy Principle of Mutation} stipulates
    that $P'$ should be the successor to $P$ if and only if $S_{P'}(e') > S_P(e)$.
The principle is \q{greedy} since
    the mutation can (and usually does) change the scores of
        expressions that are not $e$ and
        yet we only focus on how the score of $e$ changes.

Notice how the condition is
    $S_{P'}(e') > S_P(e)$ as opposed to $S_{P'}(e') \geq S_P(e)$.
This is a deliberate choice.
It means that we prefer doing no mutation to mutating:
    we only accept mutations that strictly improve the expression they change.
We shall refer to this property as \emph{mutation conservativity}.

\section{Cyclic Evolution}
\label{sec:cyclic}

We have described $\ysuppose$-expressions
    but we have yet to explain
        how $\ysuppose$ should be introduced during evolution and
        when these should be promoted into $\yif$-conditionals.
When mutating a program $P$,
    we would like to sometimes replace an expression
    \[
        e \quad \becomes \quad \xsuppose \xtrue \xthen e \xend
    \]
    after which mutations can begin to improve $\xtrue$ to a more meaningful condition.
If we were to just follow the simplistic scheme set out in \cref{sec:principles},
    such replacement might be tried
        but the resulting program will never succeed the existing program
    because the score of the replacing expression is the same as
        the score of the original expression $e$.
One could circumvent this issue
    by making an exception to mutation conservativity,
    making this specific type of mutation always succeed.
This, however, is not a very good solution to the problem at hand:
    the program could then grow unbounded
        as we add more and more redundant $\ysuppose$-expressions.

What we propose instead is to divide the evolution into
    phases of growth,
    phases of mutation, and
    phases of reduction;
    similar to how a person
        who is trying to gain muscle
        may divide their eating habits into phases of bulking and phases of cutting.
Specifically, instead of proceeding one mutation at a time,
    the basic unit of the algorithm will be a \emph{cycle}
    consisting of four phases:
    \begin{enumerate}[(1)]
        \item
        \emph{Stretching} in which a program is
            replaced by an equivalent, more redundant program
            by introducing
                redundant $\ysuppose$-expressions,
                redundant $\yif$-conditionals,
                redundant $\ycase$-expressions, and
                redundant new functions.
        \item
        \emph{Mutation} in which a number small changes are made to the program,
            replacing the program
                whenever a mutation satisfying the Greedy Principle of Mutation is found.
        \item
        \emph{Rewinding} in which we try to undo the stretches done in (a).
        \item
        \emph{Compression} in which we try to
            rewrite the program to an equivalent, smaller form.
    \end{enumerate}
The algorithm we propose evolves programs by
    repeatedly applying cycles until a program
        performing perfectly (in accord with the test criterion) is found.

We describe basic ways to implement each phase.
We emphasize that each can be extended and refined in many ways,
    some of which will be discussed in \cref{sec:future}.

\subsection{Stretching}
\label{subsec:stretch}
The stretching phase makes the program more redundant,
    setting the scene for the mutation phase
        to turn that redundancy into useful structure.
The stretching phase consists of $N_\f{stretch}$ \emph{stretches}
    for some fixed constant $N_\f{stretch} > 0$.
Each stretch picks out a random expression in the program
    with the probability of choosing an expression $e$
        being proportional to $1 - S(e)$.
The stretched program $P'$ is constructed from $P$
    by picking randomly among the following possibilities:
    \begin{enumerate}[(1)]
        \item
        replacing
            $e \becomes \xsuppose \xtrue \xthen e \xend,$
        \item
        if $e = \xsuppose c \xthen e' \xend$ and $S(c) > S(e')$,
            replacing $e \becomes \xif c \xthen e' \xelse e' \xend$,
        \item
        picking a variable $x : \tau$ bound in the context of $e$
            and replacing
            \begin{align*}
                e
                \becomes{}
                &\xcase x \xof
                    \\&
                    c_1(x_1, \ldots, x_{n_1})
                    \to
                    e[x/c_1(x_1, \ldots, x_{n_1})];
                    \\&
                    \cdots;
                    \\&
                    c_k(x_1, \ldots, x_{n_k})
                    \to
                    e[x/c_k(x_1, \ldots, x_{n_k})]
                    \\
                &\xend,
            \end{align*}
            where
                (a) $x_1, \ldots$ are variables that do not participate in $e$,
                (b) $c_1, \ldots, c_k$ are the value constructors for $\tau$:
                    \[
                        \tau \Coloneqq
                            c_1(\cdots \text{($n_1$ arguments)})
                            \mid
                            \cdots
                            \mid
                            c_k(\cdots \text{($n_k$ arguments)}),
                    \]
                and (c) $e[x/e']$ denotes $e$ with all occurrences of $x$ replaced by $e'$, or
        \item
        for
            $
                x_1 : a_1, \ldots, x_n : a_n
            $
            the variables bound in the context of $e : b$,
        adding a new function
            \begin{align*}
                &f :
                    (b, a_1, \ldots, a_n)
                    \to
                    b
                \\
                &f(y, x_1, \ldots, x_n) = y
            \end{align*}
        and replacing
            $
                e \becomes f(e, x_1, \ldots, x_n).
            $
    \end{enumerate}
In each case, the replacing expression
    should be tagged with some tags disjoint from those used in $P$,
    but we elide this for notational simplicity.

Let us elaborate on the condition
    $
        S(c) > S(e')
    $
    in stretch (2).
The justification for this condition comes from \cref{lem:suppose_true}
    which precludes stretch (2) from happening when $c = \xtrue$.
In this way,
    when a stretch (1) introduces a $\ysuppose$-expression,
    the condition must first improve
        (which can happen in the mutation phase)
        before a stretch (2) can promote it into an $\yif$-conditional.
As such, we do not end up immediately creating a lot of redundant $\yif$-conditionals.

\subsection{Mutation}
The mutation phase consists of $N_\f{mutate}$ mutations, successful or not,
    where $N_\f{mutate} > 0$ is some fixed constant.
A mutation changes the program in the way described \cref{subsec:mutex}.
A mutation is considered successful if it satisfies
    the Greedy Principle of Mutation (see \cref{subsec:greedy})
    in which case the mutant program replaces the program and the search continues.

\subsection{Rewinding}
The purpose of stretching was to create new program structures
    that the mutation phase could make use of.
However, not all stretches create useful program structures.
If we return to our informal example of \cref{sec:suppose} with evolving the $\min$ function,
    we could have the stretch
    \begin{align*}
        &\min(n, m) = \xif n \leq m \xthen n \xelse n \xend
        \\
        \becomes{}&
        \min(n, m) = \xif n \leq m \xthen n \xelse (\xsuppose \xtrue \xthen n \xend) \xend
    \end{align*}
But that would not be of much use
    because in the $\yelse$-branch $n$ is never the right expression.
So, unless the last $n$ is mutated,
    the condition for the $\ysuppose$-expression will just converge to $\xfalse$.
We must be able to deal effectively with such \q{useless stretches}
    to stop the program that is being evolved from accumulating redundancy.

After the mutation phase has completed,
    we propose to clean up the stretches
    by \q{rewinding} them.
Specifically, for each stretching operation $s$,
    we define a corresponding \emph{rewinding operation} $\overline s$,
    which is a transformation of programs that undoes
        (is left-inverse to)
        $s$.
Then, if stretches
    $
        s_1, \ldots, s_{N_\f{stretch}}
    $
    were made in the stretching phase in that order,
    the rewinding phase applies rewinding operations
    $
        \overline s_{N_\f{stretch}}, \ldots, \overline s_1
    $
    in that order
    (notice the reversal of order).

For a stretch $s$ out of the $N_\f{stretch}$ stretches done in the stretching phase,
    we will now describe its corresponding rewinding operation $\overline s$.
$\overline s$ depends on which type (1)-(4) of stretch $s$ is:
    \begin{itemize}
        \item[$(\overline 1)$]
        If $s$ is a stretch of type (1),
            $\overline s$ is the following program transformation.
        Let $e_1, \ldots, e_n$ be the $\ysuppose$-expressions
            that originated from $s$.
        Let
            $c_1, \ldots, c_n$ be their respective condition expressions and
            $e_1', \ldots, e_n'$ their respective bodies.
        For $i = 0, \ldots, n$,
            replace $e_i \becomes e_i'$ if $S(c_i) \leq S(e_i)$.
        \item[$(\overline 2)$]
        If $s$ is a stretch of type (2),
            $\overline s$ is the following program transformation.
        Let $e_1, \ldots, e_n$ be the $\yif$-conditionals
            that originated from $s$.
        Let
            $c_1, \ldots, c_n$ be their respective condition expressions and
            $e_1', \ldots, e_n'$ their respective $\xtrue$-branches.
        For $i = 0, \ldots, n$,
            replace $e_i \becomes \xsuppose c_i \xthen e_i' \xend$
            if the latter expression attains a score in the program with this replacement made
                at least the score $S(e_i)$.
        \item[$(\overline 3)$]
        If $s$ is a stretch of type (3),
            $\overline s$ is the following program transformation.
        Let $x$ be the variable chosen in $s$.
        Let $e_1, \ldots, e_n$ be the expressions $\xcase x \xof \cdots$
            that originated from $s$.
        Let
            $
                c_1(x_1, \ldots, x_{n_1}),
                \ldots,
                c_k(x_1, \ldots, x_{n_k})
            $
            be the cases that appear in each of those $\ycase$-expressions.
        For $i = 1, \ldots, n$,
            do the following.
        \begin{enumerate}[(a)]
            \item
            Look for a $1 \leq j \leq k$ such that
                (i) $n_j > 0$,
                (ii) the expression
                    $
                        e_{i, j}' \defeq
                        e_{i, j}[c_j(x_1, \ldots, x_{n_j})/x],
                    $
                    where $e_{i, j}$ is the body of the $\ycase$-expression $e_i$ to the case $c_j$,
                    does not reference the variables $x_1, \ldots, x_{n_j}$, and
                (iii) after replacing $e_i$ with $e_{i, j}'$,
                    the new $e_{i, j}'$ attains a score at least the score $S(e_i)$.
            If such a $j$ exists,
                replace $e_i \becomes e_{i, j}'$ in the program.
            \item
            If not, look instead for a $1 \leq j \leq k$ such that
                (i) $n_j = 0$ and
                (ii) after replacing $e_i$ with $e_{i, j}$,
                    the new $e_{i, j}$ attains a score at least the score $S(e_i)$.
            If such a $j$ exists,
                replace $e_i \becomes e_{i, j}$ in the program.
            \item
            If such a $j$ could not be found either, do nothing.
        \end{enumerate}
        \item[$(\overline 4)$]
        If $s$ is a stretch of type (4),
            $\overline s$ is the following program transformation.
        Let $e_1, \ldots, e_n$ be the call-expressions
            that originated from $s$.
        Let $f$ be the name of the function they are calling.
        Suppose its definition in the program is
            $
                f(x_1, \ldots, x_k) = e.
            $
        If $e$ contains a subexpression $f(\cdots)$, do nothing.
        Otherwise, do the following.
        For each $i = 1, \ldots, n$,
            replace $e_i = f(a_1, \ldots, a_k) \becomes e[x_1/a_1]\cdots[x_k/a_k]$.
        If after this transformation the expression $f(\cdots)$ appears nowhere in the program,
            furthermore remove the function definition
            $
                f(x_1, \ldots, x_k) = e
            $
            from the program.
    \end{itemize}
As before, we elide tags in notation,
    implicitly tagging expressions with unique tags.
And again, $e[a/b]$ denotes the expression $e$ with all occurrences of $a$ replaced by $b$.

Note that in each of these rewinding operations,
    we talk of possibly multiple expressions $e_1, \ldots, e_n$.
This is because an expression created by one stretch may be duplicated by another
    (stretches of type (2) or (3)).
It is therefore necessary to track
    what expressions originate in what stretches.
This can be done by tagging each expression by some extra bookkeeping data
    of which we shall not deal with the details.

The condition $S(c_i) \leq S(e_i)$ in $(\overline 1)$
    is again justified by \cref{lem:suppose_true}.
In particular, the condition is only fulfilled
    if the condition $c_i$ is at least as bad as the initial condition $\xtrue$.
In effect, we throw away a $\ysuppose$
    if its condition did not manage to improve at all during the mutation phase.

\subsection{Compression}
\label{subsec:compression}
The final phase is compression,
    replacing the program by an equivalent, shorter form.
There are indeed many ways to approach this problem of removing excess;
we present one which we found works reasonably well in practice.

As input for the algorithm,
    we need to run the program on the examples in the test criterion
    that we fixed in \cref{subsec:local_score_algo}
        (see step (a))
    and collect for each expression $e$ an \emph{input--output table}:
        the set with an element $(\sigma, \f{untag}(v))$
            for each evaluation $\sigma \vdash e \downarrow v$ done.

We say that an expression $e$ is \emph{linear} if either
    $e = x^{(t)}$ for $x$ a variable or
    $e = c(e_1, \ldots, e_n)^{(t)}$ for
        $c$ a value constructor and
        $e_1, \ldots, e_n$ linear expressions.

Recall that we defined a program to consist of
    a list of function declarations $f(x_1, \ldots, x_n) = e$.
The compression algorithm first applies the following procedure to each of the bodies $e$.
\begin{enumerate}[(1)]
    \item
    If $e = c(e_1, \ldots, e_n)^{(t)}$ for $c$ a value constructor,
        apply the procedure recursively to $e_1, \ldots, e_n$.
    \item
    If $e = \xsuppose c \xthen b \xend^{(t)}$,
        apply the procedure recursively to $c$ and $b$;
        and if after doing so
            $e$ became
            $e' = \xsuppose c' \xthen b' \xend^{(t)}$
            where $c'$ admits no variable $x^{(u)}$ as a subexpression,
            further replace $e' \becomes b'$.
    \item
    If neither of the above cases apply, try to find
        a linear expression $\ell$ with the behavior specified by the input--output table of $e$.
    (The algorithm for finding $\ell$ is postponed to \cref{subsubsec:lexprs}.)
    If such $\ell$ exists,
        pick one and
        replace $e \becomes \ell$
            (in which we tag $\ell$ with expression tags disjoint from those used in $P$),
        and we're done.
    If not,
        continue to the next steps.
    \item
    If $e = \xif c \xthen a \xelse b \xend^{(t)}$,
        apply the procedure recursively to $c$, $a$, and $b$.
    Suppose that after doing so
        $e$ became
        $e' = \xif c' \xthen a' \xelse b' \xend^{(t)}$.
    If $c' = \xtrue^{(u)}$ for some $u$,
        further replace $e' \becomes a'$.
    Or, if $c' = \xfalse^{(u)}$ for some $u$,
        further replace $e' \becomes b'$.
    \item
    If $e = f(e_1, \ldots, e_n)^{(t)}$ for $f$ a function,
        apply the procedure recursively to $e_1, \ldots, e_n$.
\end{enumerate}
After this, the algorithm removes from the program function definitions
    $
        g(\cdots) = \cdots
    $
    for functions $g$ that are not reachable from the primary function $f$ in the call graph.
The resulting program is the result of the compression phase.

Note that, in each step
    the $w_k$ used in
        the definition of $S$ in \cref{subsec:local_score_algo}
    remain unchanged if we disregard traces.
This is the sense in which the original program is equivalent to the compressed program,
    that is,
    they are equivalent on known input.

\subsubsection{An algorithm for finding linear expressions}
\label{subsubsec:lexprs}
We are presented with the problem of
    listing the linear expressions $\ell$ satisfying a given specification.
Let us be more precise.
We are given
    a type $\tau$,
    a typing context $\Gamma$ containing a definition of $\tau$, and
    an input--output table $T = \{(\sigma_1, v_1), \ldots, (\sigma_n, v_n)\}$
such that for each $i = 1, \ldots, n$,
    (i) $\sigma_i$ is an evaluation context and $v_i$ an untagged value of type $\tau$ and
    (ii) all variables found in $\sigma_i$ are found in $\Gamma$ and vice versa.
The task is then to compute the (untagged) linear expressions $\ell$
    that evaluate\footnote{
        according to semantics like those given in \cref{subsec:language}
            but ignoring traces
    }to $v_i$ in context $\sigma_i$ for every $i = 1, \ldots, n$.

If $n = 0$ (that is, the input--output table is empty),
    the problem is equivalent to finding all linear expressions of type $\tau$.
This possibly infinite list can easily be computed by a breadth-first search
    applied to the recursive definition of linear expressions
        given near the beginning of \cref{subsec:compression}.

Now, suppose $n > 0$.
For convenience,
    we shall rename the variables in $\Gamma$
        so that they are called $x_1, \ldots, x_n$.
Recall the notation for tagged values established in \cref{subsec:local_score_algo}.
An \emph{$\ell$-value} is a tagged value $\in v^{\mathscr P(\{\circ, 1, \ldots, n\})}$
    where $\mathscr P$ denotes the power set operator.
Moreover, given an $\ell$-value $w = c^A(w_1, \ldots, w_k)$,
    define
    \[
        \ell(w)
        \defeq
        \{x_i^{()} \mid i \in \N \cap A\}
        \cup
        \begin{cases}
            \{c^{()}(w_1', \ldots, w_k') \mid w_i' \in \ell(w_i)\},
                & \text{if $\circ \in A$;} \\
            \emptyset,
                & \text{otherwise.}
        \end{cases}
    \]

\begin{lemma}
    $\ell(w)$ is a finite set.
\end{lemma}

\begin{proof}
    Proceed by induction over $w$:
        assume that the lemma is true
            when $w$ is replaced by any of the $w_i$.
    Since $A$ is finite,
        the induction step follows from the formula defining $\ell(w)$.
\end{proof}

To an input--output pair $(\sigma, v)$
    where
    $
       v = c^{()}(v_1, \ldots, v_k),
    $
    we associate an $\ell$-value
        \[
            w(\sigma, v)
            \defeq
            c^{\{\circ\} \cup \{i \mid (x_i = v) \in \sigma\}}
                (w(\sigma, v_1), \ldots, w(\sigma, v_k)).
        \]

\begin{lemma}
    $\ell(w(\sigma, v))$ is the set of linear expressions
        that evaluates to $v$ in the evaluation context $\sigma$.
\end{lemma}

\begin{proof}
    Proceed by induction on $v$:
        assume that the lemma is true
            when $v$ is replaced by any of the $v_i$.
    The linear expressions evaluating to $v$ in the context $\sigma$ are
        the variables $x_i$ such that $(x_i = v) \in \sigma$ and
        expressions $c(e_1, \ldots, e_k)$
            where each $e_i$ is a linear expression evaluating to $v_i$ in the context $\sigma$.
    The induction step follows from unfolding definitions.
\end{proof}

Given $\ell$-values
    $w_1 = c_1^A(w_{1, 1}, \ldots, w_{1, p})$ and
    $w_2 = c_2^B(w_{2, 1}, \ldots, w_{2, q})$,
    define
    \[
        w_1 * w_2
        \defeq
        \begin{cases}
            c_1^{A \cap B}(w_{1, 1} * w_{2, 1}, \ldots, w_{1, p} * w_{2, p}),
                & \text{if $c_1 = c_2$;} \\
            c_1^{A \cap B \setminus \{\circ\}}(w_{1, 1}, \ldots, w_{1, p}),
                & \text{otherwise.}
        \end{cases}
    \]

\begin{lemma}
    $\ell(w_1 * w_2) = \ell(w_1) \cap \ell(w_2)$.
\end{lemma}

\begin{proof}
    Proceed by induction on $w_1$:
        assume that the lemma is true
            when $w_1$ is replaced by any of the $w_{1, i}$.

    \para{Case: $c_1 = c_2$}
    Using the induction hypothesis,
        \begin{align*}
            \ell(w_1 * w_2)
            &=
            \ell(c_1^{A \cap B}(w_{1, 1} * w_{2, 1}, \ldots, w_{1, p} * w_{2, p}))
            \\
            &=
            \{x_i^{()} \mid i \in \N \cap A \cap B\}
            \\
            &\quad \cup
            \begin{cases}
                \{c_1^{()}(w_1', \ldots, w_n') \mid w_i' \in \ell(w_{1, i} * w_{2, i})\},
                    & \text{if $\circ \in A \cap B$;} \\
                \emptyset,
                    & \text{otherwise.}
            \end{cases}
            \\
            &=
            \{x_i^{()} \mid i \in \N \cap A \cap B\}
            \\
            &\quad\cup
            \begin{cases}
                \{c_1^{()}(w_1', \ldots, w_n') \mid w_i' \in \ell(w_{1, i}) \cap \ell(w_{2, i})\},
                    & \text{if $\circ \in A \cap B$;} \\
                \emptyset,
                    & \text{otherwise.}
            \end{cases}
            \\
            &{{}=}
            \(
            \{x_i^{()} \mid i \in \N \cap \mask{B}{A}\}
            \cup
            \begin{cases}
                \{c_1^{()}(w_1', \ldots, w_n') \mid w_i' \in \ell(w_{1, i})\},
                    & \text{if $\circ \in A$;} \\
                \emptyset,
                    & \text{otherwise.}
            \end{cases}
            \)
            \\
            &\mask{{}=}{\cap}
            \(
            \{x_i^{()} \mid i \in \N \cap B\}
            \cup
            \begin{cases}
                \{c_1^{()}(w_1', \ldots, w_n') \mid w_i' \in \ell(w_{2, i})\},
                    & \text{if $\circ \in B$;} \\
                \emptyset,
                    & \text{otherwise.}
            \end{cases}
            \)
            \\
            &= \ell(w_1) \cap \ell(w_2).
        \end{align*}

    \para{Case: $c_1 \neq c_2$}
    Unfolding definitions,
        \begin{align*}
            \ell(w_1 * w_2)
            &=
            \ell(c_1^{A \cap B \setminus \{\circ\}}(w_{1, 1}, \ldots, w_{1, p}))
            \\
            &=
            \{x_i^{()} \mid i \in \N \cap A \cap B\}
            \cup
            \emptyset
            \\
            &{{}=}
            \(
            \{x_i^{()} \mid i \in \N \cap \mask{B}{A}\}
            \cup
            \begin{cases}
                \{c_1^{()}(w_1', \ldots, w_n') \mid w_i' \in \ell(w_{1, i})\},
                    & \text{if $\circ \in A$;} \\
                \emptyset,
                    & \text{otherwise.}
            \end{cases}
            \)
            \\
            &\mask{{}=}{{\cap}}
            \(
            \{x_i^{()} \mid i \in \N \cap B\}
            \cup
            \begin{cases}
                \{c_2^{()}(w_1', \ldots, w_n') \mid w_i' \in \ell(w_{2, i})\},
                    & \text{if $\circ \in B$;} \\
                \emptyset,
                    & \text{otherwise.}
            \end{cases}
            \)
            \\
            &=
            \ell(w_1) \cap \ell(w_2).
        \end{align*}
\end{proof}

We can now describe the algorithm for finding linear expressions.
Combining the lemmata above,
    we find that the set
    \[
        \ell(w(\sigma_1, v_1) * \cdots * w(\sigma_n, v_n))
    \]
    contains exactly the linear expressions coforming to the input--output table $T$.
Since the definitions of $\ell$, $w$, and $*$ deal with
    finite data in finite numbers of steps,
    we can compute this set in its entirety by unfolding definitions,
    thus giving an algorithm.

\section{A proof-of-concept implementation}
\label{sec:poc}

We provide a proof-of-concept implementation
    which can be found at \url{https://gitlab.com/maxvi/ssgp}.
Excluding comments and blank lines,
    it consists of a little less than 2700 lines of Haskell code.
Great care was taken to document the API and every piece of code,
    amounting to over 1300 lines of comments.
Various QuickCheck tests are also included.
We emphasize that the implementation is a \emph{minimum viable product};
    for instance, it is completely unoptimized
        (see \cref{subsec:optimize}).

The implementation deviates from a naïve implementation of the theoretical ideas described above
    in a number of ways,
    many of which we shall cover in this section.
It produces code written in
    what is, sparring $\ysuppose$-expressions, a subset of Haskell.

\subsection{Results}
\label{subsec:results}
We have chosen a few examples, tabulated below,
    which highlight some strengths and weaknesses of our implementation.

\begin{figure}
    \begin{tabular}{| l | l | l | c | r |}
        \hline
        Name
            & Signature
                & Avail.~funcs.
                    & Convergence
                        & Avg.~iter.
        \\
        \hline
        $\f{sum}$
            & $[\Z] \to \Z$
                & $+ : (\Z, \Z) \to \Z$
                    & $5/5 = 100\%$
                        & 149.4
        \\
        $\f{min}$
            & $(\Z, \Z) \to \Z$
                & ${\leq} : (\Z, \Z) \to \xbool$
                    & $5/5 = 100\%$
                        & 373.6
        \\
        $\f{min}'$
            & $[\Z] \to \Z$
                & ${\leq} : (\Z, \Z) \to \xbool$
                    & $5/5 = 100\%$
                        & 1076.4
        \\
        $\f{insert}$
            & $(\Z, [\Z]) \to [\Z]$
                & ${\leq} : (\Z, \Z) \to \xbool$
                    & $5/5 = 100\%$
                        & 1485.6
        \\
        $\f{helloworld}$
            & $() \to [\Z]$
                &
                    & $5/5 = 100\%$
                        & 2611.0
        \\
        $\f{sort}$
            & $[\Z] \to [\Z]$
                & ${\leq} : (\Z, \Z) \to \xbool$
                    & $0/5 = \ \ \ 0\%$
                        & NA
        \\
        $\f{reverse}$
            & $[\Z] \to [\Z]$
                &
                    & $0/5 = \ \ \ 0\%$
                        & NA
        \\
        \hline
    \end{tabular}
    \caption{Some benchmarks of our proof-of-concept implementation.}
\end{figure}

\begin{figure}
    \begin{minted}{haskell}
    insert :: Int -> [Int] -> [Int]
    insert x0 (Cons[Int] x1 (Cons[Int] x2 x3)) =
        if f0 x0 x1 (Cons[Int] x2 x3) then
            Cons[Int] x1 (insert x0 (Cons[Int] x2 x3))
        else
            Cons[Int] x0 (Cons[Int] x1 (Cons[Int] x2 x3))
    insert x0 Nil[Int] = Cons[Int] x0 Nil[Int]
    insert x0 (Cons[Int] x1 Nil[Int]) =
        if f0 x0 x1 Nil[Int] then
            Cons[Int] x1 (Cons[Int] x0 Nil[Int])
        else
            Cons[Int] x0 (Cons[Int] x1 Nil[Int])

    f0 :: Int -> Int -> [Int] -> Bool
    f0 x0 x1 x2 = leq x1 x0
    \end{minted}
    \caption{Example of code produced by our proof-of-concept implementation.}
\end{figure}

Each example was run five times.
Each run was set to have a maximum of 10 cycles (see \cref{sec:cyclic}),
    after which we would declare the evolution to not converge.
The constants used were
    $N_\f{stretch} = 3$ and
    $N_\f{mutate} = 300$.
The last column records the average number of
    iterations (stretches, mutations, rewindings, compressions) that occured
        before evolution converged to correct code.
At each iteration,
    a new program is proposed and
    this program is then run on a pre-fixed set of examples
        through which local scores are assigned.

Every type is an Algebraic Data Type.
Integers are encoded by their binary representation.
Lists are defined recursively in a left-biased way as
    $
        [\Z] \Coloneqq \mathbf{Cons}(\Z, [\Z]) \mid \mathbf{Nil}.
    $
There is no standard library;
    each example can only use the functions specified in the third column,
    nothing else.
If one were to add more useful functions there,
    one would get better convergence rates.

Let us elaborate on the behavior of some of the tabulated functions.
\q{$\min'$} picks out the smallest number from a list,
    defaulting to $0$ if the list is empty.
\q{$\f{insert}$} inserts a number into a list of numbers before the first entry
    that is greater than or equal to the number.
\q{$\f{helloworld}$} lists the ASCII codes for the string \q{Hello World}.
The rest are self-explanatory.

\subsection{Comparison with other works}
\label{subsec:compare}
Koza \cite{koza94} established the \emph{even-parity problem}
    as a benchmark for Genetic Programming.
The objective is to evolve code for the function $[\xbool] \to \xbool$
    that returns
        $\xtrue$ if the list has even number of $\xtrue$s and
        $\xfalse$ otherwise.
Briggs and O'Neill \cite{briggs06} use this problem
    to compare a number of approaches found in literature.
We reproduce their benchmarks below,
    adding a benchmark of our own proof-of-concept implementation.

\begin{figure}
    \begin{tabular}{| l | l | r |}
        \hline
        Approach
            & Citation
                & Assessments
        \\
        \hline
        GP with Local Scoring
            & This paper
                & 882
        \\
        Exhaustive enumeration
            & \cite{briggs06}
                & 9478
        \\
        PolyGP
            & \cite{yu99}
                & 14,000
        \\
        GP with combinators
            & \cite{briggs06}
                & 58,616
        \\
        GP with iteration
            & \cite{kirshenbaum01}
                & 60,000
        \\
        Generic GP
            & \cite{wong96}
                & 220,000
        \\
        OOGP
            & \cite{agapitos06}
                & 680,000
        \\
        GP with ADFs
            & \cite{koza94}
                & 1,440,000
        \\
        \hline
    \end{tabular}
    \caption{
        Comparison of the smallest number of program assessments needed
            for $99\%$ convergence rate on the even-parity problem
            across various approaches.
        Smaller numbers are better.
        (No claims are made as to exhaustiveness of the listed approaches.)
    }
\end{figure}
The last column records estimates for the smallest number of program assessments
    needed during evolution
    for an algorithm to find a correct solution with 99\% probability.
A \emph{program assessment} scores a program
    by running it on a number of examples.

We ran our proof-of-concept implementation 100 times on the even-parity problem.
We used a Lenovo ThinkPad X1 Carbon with an Intel Core i7-8550U CPU.
It took approximately 11 minutes in total.
Every run eventually converged on correct code,
    with the average number of programs tested before convergence being 398.83.
The best run needed to test just 68 programs and the worst 940 programs.
Generalizing from our data,
    we estimate that if we stop evolution after 882 programs,
    we arrive at a correct solution with $99\%$ probability.
The only built-in function made available from the onset was the binary operation,
    $
        \f{not} : \xbool \to \xbool.
    $


\subsection{Local scores and deep recursion}
For scores $S(e)$,
    we use instead a weighted average in which
        each summand is weighted by the reciprocal of the call-stack depth.
In effect, deep recursive calls have less influence on scores.

\subsection{Continuous scoring}
\label{subsec:continuous}
The Hello World example uses a \q{continuous} scoring function,
    meaning that it takes into account
        how close the string that $\f{helloworld}()$ produces
        is to the desired string \q{Hello World},
    instead of just rewarding $1$ when correct and $0$ when incorrect
        (in which case evolving the correct code would be virtually impossible).
Local scoring means that if $\f{helloworld}()$ returned the string \q{Hello Gorld},
    we trace the problematic character, G, back to the expression that produced it
    and then change that expression little by little,
        inching in on the correct character, W.

\subsection{Timeouts and restrictions to recursion}
\label{subsec:rec}
A fundamental issue is that the problem of bounding the running time of
    an arbitrary program in a Turing-complete language
    is undecidable
        as was famously shown by Turing in 1936 \cite{turing36}.
To overcome this issue,
    we imposed restrictions on what is allowed during execution:
\begin{enumerate}[(1)]
    \item
    We impose a limit on
        how many function calls can be made in one program execution.
    If exceeded,
        the local score of
            every expression that was evaluated during the execution
            gets punished.
    \item
    We impose restrictions on what forms of recursions are allowed.
    Specifically, we shall only allow a call $f(v_1, \ldots, v_n)$
        if there is an $i$
        such that for all calls $f(w_1, \ldots, w_n)$ on the call-stack,
            $|v_i| < |w_i|$.
    Here, $|v|$ denotes the \q{size} of a value $v$,
        that is, its number of nodes.
    If this condition is violated,
        the local score of every expression on the \q{expression stack}
            (the expression-level analogue of the call-stack)
            gets punished.
\end{enumerate}
The restriction (2) alone ensures termination,
    as the following lemma shows.

\begin{lemma}
    There is no infinite sequence $a_1, a_2, \ldots \in \N^n$
        with the property that each $a_i$ has a coordinate $a_{i, j} \in \N$
            such that $a_{i, j} < a_{i', j}$ for any $i' < i$.
\end{lemma}

\begin{proof}
    Assume for contradiction that such sequence $a_*$ exists.
    The sequence may be partitioned into $n$ subsequences,
        $
            p_{1, *}, \ldots, p_{n, *}
        $
        such that for each $i = 1, \ldots, n$,
        $
            p_{i, k, i} < p_{i, k', i}
        $
        for any $k' < k$.
    Since $a_*$ is infinite,
        one of these subsequences, say $p_{j, *}$, is infinite.
    But then
        $
            p_{j, 1, j} > p_{j, 2, j} > \cdots
        $
        is an infinite descending sequence in $\N$,
        which is not possible.
\end{proof}

\subsection{Controlled Recursion}
\label{subsec:controlled}
When a recursive function is broken,
    it is usually catastrophically broken.
This is because errors propagate (and multiply) upwards through recursion.
This affects negatively the quality of the local scores
    which we want to reflect \q{closeness} to a solution.
For example, if we were to define natural numbers inductively using successors,
    $
        \N \Coloneqq 0 \mid S(\N),
    $
    then the recursive definition
    \begin{align*}
        0 \oplus m &\defeq m
        \\
        S(n) \oplus m &\defeq n \oplus m
    \end{align*}
    is only one $S$ away from being the correct definition of addition
        ($S(n) \oplus m = S(n \oplus m)$),
    yet in terms of actual behavior it is quite far away:
        $n \oplus m = m$.
To deal with this issue,
    we employ a technique we shall call \emph{Controlled Recursion},
    in which we answer recursive calls using training data.
For example, if we were trying to evolve $\oplus$ to be addition,
    we would,
    when producing the values $w_1, \ldots, w_N$ in \cref{subsec:local_score_algo},
    evaluate the recursive call $n \oplus m$
        in the body of the second equation above as $n + m$
        instead of using the actual code for $\oplus$.
That means that the output only differs by $1$ from the desired output,
    and the score is accordingly better
        (that is, if we use a continuous scoring function as in \cref{subsec:continuous})
        than if we were to answer recursive calls using the actual code.
To discourage degenerate self-referential definitions like
    $
        S(n) \oplus m \defeq S(n) \oplus m,
    $
    we only answer recursive calls like this that are \q{smaller}
        than the calls on the call-stack,
        a la the restriction (2) we imposed in \cref{subsec:rec}.

\subsection{Stretches}
Since the examples did not appear to need it,
    we disabled the stretch (4) described in \cref{subsec:stretch}.
Moreover, for stretch (1),
    we only introduce $\ysuppose$-expressions
        at the top of bodies of equations.\footnote{
    It would probably be better to allow them to be introduced wherever
        but have a mechanism pushing them up the expression tree (\q{generalization})
            before they are promoted into $\yif$-conditionals
                so as to not have to reinvent the same conditions for various subexpressions.
}

\subsection{Limitations}
A significant issue is the performance suboptimality of our proof-of-concept implementation.
With optimized code execution,
    we could perhaps try many times more mutations in the same amount of time,
    thereby searching much deeper for new mutations.
We discuss this further in \cref{subsec:optimize}.

Another striking issue is that of overfitting.
For example, overfitting obstructs discovery of correct code for the function
    $
        \f{reverse} : [\Z] \to [\Z]
    $
    reversing the items of a list.
Instead of attacking the problem algorithmically,
    it generates more and more cases:
        one for the empty list $[]$,
        one for $[x_1]$,
        one for $[x_1, x_2]$,
        and so on.
While it obviously does get closer to a correct solution in terms of behavior,
    it does not approach a correct solution in terms of code.
The issue is that lists are defined recursively in a left-biased way,
    making operations on the beginning of the list far easier to discover
        than operations on the end.

A human is able to immediately tell that such code is bad style;
    perhaps if one taught the same intuition to a neural network,
        that could be used to stir away from overfitting.
See \cref{subsec:ml} for more discussion.

Overfitting was not an issue
    in any of the converging examples from \cref{subsec:results}.

\section{Extensions and future work}
\label{sec:future}

\subsection{Crossover and parallelism}
In contrast to a lot of the literature (see \cref{subsec:compare}),
    we use what is a kind of hill-climbing algorithm:
    at every point in the algorithm,
        we have exactly one program.
It is typical to instead evolve a population of programs,
    mutating them and using a procedure called \emph{crossover}
        to combine two programs producing one or more offspring programs.
We remark that
    this could be added on top of our hill-climbing method,
    making for highly parallelized Genetic Programming:
        one could do $C$ cycles of hill-climbing (see \cref{sec:cyclic}) for each program in the population
        $
            \{P_{(1)}, \ldots, P_{(n)}\}
        $
        in parallel,
        yielding a new population
        $
            \{P_{(1)}', \ldots, P_{(n)}'\}
        $
        which is then extracted for crossover
            combining programs at random and
            killing off all but the $n$ best resulting programs.
This procedure is repeated until a solution is found.
With this, many hill-climbing processes run in parallel but
    still \q{share discoveries} once in a while,
    avoiding too much duplicated effort.

\subsection{Optimizations}
\label{subsec:optimize}
By far the most significant component performance-wise is
    that which executes programs.
As we were more focused on the theoretical background,
    no serious attempt at optimizing this was made
        in our proof-of-concept implementation written in Haskell.
Let us discuss some techniques and ideas
    that could be used in a performant implementation.

\begin{enumerate}[(1)]
    \item
    Use a more performant language such as Rust or C.
    \item
    Currently, we represent code as an Abstract Syntax Tree
        which is quite inefficient for a number of reasons.
    For instance, execution requires dereferencing a lot of pointers
        resulting in many cache misses.
    A more efficient, flat representation of code,
        such as a stack machine,
        could be used instead.
    \item
    Considering that the same program is executed many times
        ranging over the different test cases,
        it would make sense
            to first compile the program to (optimized) machine code.
    Compilation is typically a heavy process and
        it would therefore be preferable
            if only the part of the program that changed was recompiled
                (incremental recompilation).
    \item
    Since we are running the same (or similar) code again and again
        and do possibly overlapping recursion,
        it would be beneficial to make heavy use of memoization,
        trading CPU cycles for increased memory usage.
    \item
    If the programs were instead executed with respect to a lazy evaluation strategy,
        programs that contain bloat would execute much faster.
    For example,
        calling a function and passing its return value to
            another function in an argument that is unused
            would then not have any effect on performance.
    In some instances, this would result in an exponential speedup:
    in our experiments, it was a common occurrence
        that correct but extremely slow programs would evolve quickly
        with the majority of evolution then spent on
            reducing the exponential running time that is due to unused evaluations,
        as such exponential running time hinders getting perfect scores on
            some of the bigger test cases
            due to the timeouts described in \cref{subsec:rec}.
    \item\label{item:necessary_values}
    When testing a mutation $e \becomes e'$,
        it is only necessary to recompute a $w_i$
            (see \cref{subsec:local_score_algo})
            if it contains the tag of $e$ in one of its traces,
            yet we recompute all.
    \item
    A significant amount of time is spent on testing out bad mutations,
        most of which are scoring extremely poorly.
    One could use heuristics for quickly discarding
        large swaths of bad mutations with only few false positives.
    Local Scoring is particularly apt for this.
    For example, if one was to test a mutation $e \becomes e'$ of an expression $e$,
        one could first compute the score of $e'$
            with respect to a small, randomly chosen subset of $w_1, \ldots, w_N$
                (but only selecting necessary $w_i$'s as in \cref{item:necessary_values}),
        and then preemptively discard the mutation if this \q{approximate score} is
            significantly smaller than $S(e)$.\footnote{
                Or perhaps, instead of $S(e)$, it would be better to
                    use the corresponding approximate score computed for $e$.
                This would necessitate storing not just the local scores
                    but also all summands of the averages,
                    so that averages of subsets can be computed at a later point.
            }
    \item
    While our proof-of-concept implementation does ensure
        that the exact same program is not tried multiple times,
        one could employ more advanced means of ensuring
            that equivalent mutations are not attempted multiple times.
    One would need to specify algorithms that are able to tell
        if a mutation on a program $P$ corresponds
            to a mutation on a program $P'$ equivalent to $P$ by a number of stretches
                (see \cref{subsec:stretch}).
    Perhaps a hash function could be devised for this purpose.
\end{enumerate}

\subsection{Integration with Deep Learning methods}
\label{subsec:ml}
We are choosing new code to try out essentially at random
    instead of doing so intelligently and
        drawing from an intuitive understanding of
            what changes are needed in what situations.
Recently, there has been major breakthroughs on
    the subject of using Deep Learning methods to generate code
    (e.g., \cite{chen21} and \cite{li22}).
An interesting direction would be to hybridize the approaches,
    combinining the \emph{overview} of Local Scoring with the \emph{intuition} of Deep Learning.
Let us speculate a bit on what that could look like.

\begin{enumerate}[(1)]
    \item
    One could incorporate into the scores an assessment of the \q{plausibility} of code,
        as determined by a Deep Learning model,
        rewarding code that looks like human-written code.
    This could help guide the search in directions
        of known patterns,
        combating issues like overfitting and getting stuck in local minima.
    \item
    What kind of mutation or stretch to do could be chosen using a model
        that is trained (on a fixed training set of code)
        to predict what is most likely to cause improvements.
    Deep Learning-based code generation could be used to
        propose candidates for the replacing expressions
        that we use when mutating
        instead of generating these at random.
    \item
    Additionally, since we use mutation conservativity (see \cref{subsec:greedy}),
        there will typically be long periods of attempting mutations
            before a successful mutation is found.
    One could collect data from these failed mutations.
    An idea is to then use Deep Learning to
        analyze these attempted mutations and their resulting local scores
        in order to get a sense of which mutations do well,
        which could then be leveraged to pick out qualified guesses for new mutations
            making use of problem-specific knowledge.
    Perhaps, one could even train a model during evolution on this data
        to predict what mutations are likely to be successful.
\end{enumerate}


\bibliographystyle{alpha}
\bibliography{references}

\end{document}